\documentclass{colt2015} 

\usepackage{tikz}
\usepackage{verbatim}
\usepackage{blkarray}
\graphicspath{{./Figures/}}

\newcommand{\Ecal}{\mathcal{E}}

\newcommand{\Mcal}{\mathcal{M}}

\newcommand{\R}{\mathbb{R}}

\newcommand{\F}{\operatorname{F}}
\newcommand{\Q}{\operatorname{Q}}

\newcommand{\dec}{\operatorname{int}}
\newcommand{\bin}{\operatorname{bin}}

\newcommand{\sgn}{\operatorname{sgn}}

\newcommand{\hs}{\operatorname{hs}}

\newcommand{\bo}{\boldsymbol{0}}

\newcommand{\unitlayer}[3]{ 
	\framebox{\tikzset{node distance=.4cm, auto}
		\begin{tikzpicture}[scale=0.9, every node/.style={transform shape}]
		\tikzstyle{neuron}=[circle, draw=black, inner sep=.01cm, minimum size = .55cm]
		
		\foreach \name / \i in {1,...,#2}
		\node[neuron] (I-\name) at (\i,0) {$#1_{\i}$};
		
		\node (I-dots) [node distance = .8cm, right of = I-#2] {$\cdots$};
		\node[neuron] (I-end)  [node distance = .8cm, right of = I-dots] {${#1}_{#3}$};
		
		\end{tikzpicture}
	}}

\title[Universal Approximation by Stochastic Feedforward Networks]{Universal Approximation of Markov Kernels by\\Shallow Stochastic Feedforward Networks}
\usepackage{times}

 \coltauthor{\Name{Guido Mont\'ufar} \Email{montufar@mis.mpg.de}\\
 	\addr Max Planck Institute for Mathematics in the Sciences,
 	04103 Leipzig, Germany
}
 
\begin{document}

\maketitle
\thispagestyle{empty}
\begin{abstract}
We establish upper bounds for the minimal number of hidden units for which a binary stochastic feedforward network with sigmoid activation probabilities and a single hidden layer is a universal approximator of Markov kernels. We show that each possible probabilistic assignment of the states of $n$ output units, given the states of $k\geq1$ input units, can be approximated arbitrarily well by a network with $2^{k-1}(2^{n-1}-1)$ hidden units. 
\end{abstract}


\section{Introduction}
The universal approximation capabilities of feedforward networks with one hidden layer of computational units have been studied in numerous papers and have been established under quite general conditions on the activation functions and the input-output domains~\citep{Cybenko, Hornik89, leshno1993multilayer, ChenChen1995,GallantWhite1988}. 
Some works have also studied the minimal size of universal approximators and the quality of the approximations when the networks have only a limited number of hidden units~\citep{Hornik:1991:ACM:109691.109700,Barron1993, Wenzel:2000:HAS:361159.361171}. 

In the context of feedforward networks, the universal approximation question most commonly refers to the approximation of deterministic functions. 
In this paper we address a related problem that has received a bit less attention. 
We study the universal approximation of stochastic functions (Markov kernels) and the minimal number of hidden units in a stochastic feedforward network that is sufficient for this purpose. 
For a network with $k$ input binary units and $n$ output binary units, 
we are interested in maps taking inputs from $\{0,1\}^k$ to probability distributions over outputs from $\{0,1\}^n$. 
The outputs of the network are length-$n$ binary vectors, but the outputs of the stochastic maps are length-$2^n$ probability vectors. 
We focus on shallow networks, with one single hidden layer, as the one illustrated in Figure~\ref{figure:structure}, 
and stochastic binary units with output $1$ probability given by the sigmoid of a weighted sum of their inputs. 
Given the number of input and output units, $k$ and $n$, what is the smallest number of hidden units $m$ that suffices to obtain a universal approximator of stochastic maps? We show that this is not more than $2^{k-1}(2^{n-1} -1)$. 
We also consider the case where the weights of the output layer are fixed in advance and only the weights of the hidden layer are tunable. In that setting we show that $2^{k-1}(2^n-1)$ hidden units are sufficient. 

Some previous works have discussed compact representations of stochastic maps by feedforward networks, 
but focusing on the approximation of probability distributions (instead of Markov kernels) and deep networks with many hidden layers~\citep{Sutskever:2008, LeRoux:2010:DBN:1836204.1836214,Montufar2011}. 

This paper is organized as follows. 
Section~\ref{section:settings} contains basic definitions and notations, as well as a few comments about the deterministic setting. 
Section~\ref{section:results} presents our main results. 
Section~\ref{section:proof1} contains our analysis of the minimal number of hidden units for a network can approximate any stochastic function arbitrarily well by tuning the input weights and biases of the first layer, while the weights and biases of the second layer are kept fixed. 
Section~\ref{section:proof2} contains a corresponding analysis for the case when input weights and biases of both layers are tuned. 
Section~\ref{section:conclusion} offers our conclusions and outlook.

\section{Settings}
\label{section:settings}

This section contains definitions and basic observations. 

\subsection{Probability Distributions and Markov Kernels}
Throughout this paper $k$, $m$, and $n$ denote finite natural numbers. 
We denote by $\{0,1\}^n$ the set of length-$n$ binary strings. 
This is the set of all possible configurations of $n$ binary units. 
The set of probability distributions over $\{0,1\}^n$ is given by 
\begin{equation}
	\Delta_n :=\Big\{ p \in \R^{\{0,1\}^n}\colon p(x) \geq 0, \sum_{x\in\{0,1\}^n} p(x) = 1 \Big\}.
\end{equation} 
This is the $(2^n-1)$-dimensional simplex of all vectors in $2^n$-dimensional Euclidean space with non-negative entries and $1$-norm equal to $1$. 
The set of strictly positive distributions is the relative interior of $\Delta_n$, denoted by $\Delta_n^+$. 

A Markov kernel with source $\{0,1\}^k$ and target $\{0,1\}^n$ is a map from $\{0,1\}^k$ to $\Delta_n$. 
The set of all such kernels is 
\begin{equation*}
	\Delta_{k,n} 
	:= \Big\{ P\in \R^{\{0,1\}^k\times\{0,1\}^n} \colon P(y;x) \geq 0, \sum_{x\in\{0,1\}^n} P(y;x) = 1\; \forall y\in\{0,1\}^k \Big\}. 
\end{equation*}  
Each Markov kernel is written as a matrix $P$ with entries $P(x|y)\equiv P(y;x)$ for all $x\in \{0,1\}^n$, for all $y\in\{0,1\}^k$. 
The $y$-th row $P(\cdot|y)$ is a probability distribution of the output $x$, given the input $y$. 
The set $\Delta_{k,n}$ is the $2^k(2^n-1)$-dimensional polytope of $2^k\times 2^n$ row-stochastic matrices, 
which is the \mbox{$2^k$-th} Cartesian power of $\Delta_n$.

\subsection{Feedforward Stochastic Networks}

We consider stochastic binary units of the following form. 
Consider the sigmoid function 
\begin{equation*}
	\sigma\colon \R\to [0,1]; \; a\mapsto \frac{1}{1+\exp(-a)}.  
\end{equation*}
For each input vector $y\in\{0,1\}^k$, the unit computes a scalar pre-activation value by an affine function $v^\top y + b$, 
and outputs $1$ with probability $\Pr(1)=\sigma(v^\top y + c)$ or otherwise it outputs $0$ with complementary probability, $\Pr(0)=1-\sigma(v^\top y + c)$. 

Given an input vector $y\in\{0,1\}^k$, 
the probability that a feedforward layer of $m$ stochastic units outputs $z = (z_1,\ldots, z_m)^\top \in\{0,1\}^m$ is given by the product of the output probabilities of the individual units in the layer, 
\begin{align*}
	\Pr(z) 
	=& \prod_{j\in[m]} \sigma(v_j^\top y + c_j)^{z_j} (1 - \sigma(v_j^\top y + c_j))^{1-z_j}  \\ 
	=& \frac{\exp( (V y + c)^\top z)}{\sum_{z'\in\{0,1\}^m}\exp( (V y + c)^\top z')}, 
\end{align*}
where $V = (v_1 | \cdots | v_m)^\top\in\R^{m\times n}$ and $c=(c_1,\ldots, c_m)^\top\in \R^m$. 

\begin{definition}
	We denote by $\F_{k,m}\subseteq\Delta_{k,m}$ the set of Markov kernels that can be represented by a feedforward layer with $k$ input units and $m$ output units; that is, the set of kernels 
	\begin{gather*}
		P(z|y) = \frac{\exp( (V y + c)^\top z)}{\sum_{z'}\exp( (V y + c)^\top z')}\quad\forall z\in\{0,1\}^m, y\in\{0,1\}^k, 
	\end{gather*}
	parametrized by $V\in\R^{m\times k}$ and $c\in\R^m$. 
\end{definition}

\begin{definition}
	We denote by $\F_{k,m,n} = \F_{m,n}\circ \F_{k,m} \subseteq\Delta_{k,n}$ the set of Markov kernels that can be represented 
	by a feedforward network with $k$ input units, $m$ hidden units, and $n$ output units; that is, the set of kernels of the from 
	\begin{equation*}
		P(x|y) = \sum_{z\in\{0,1\}^m} Q(z|y) R(x|z) \quad \forall y\in\{0,1\}^k, x\in\{0,1\}^n , 
	\end{equation*}
	where $Q\in\F_{k,m}$ and $R\in\F_{m,n}$. 
\end{definition}

Fig.~\ref{figure:structure} gives a schematic illustration of a feedforward network with one input, one hidden, and one output layer.

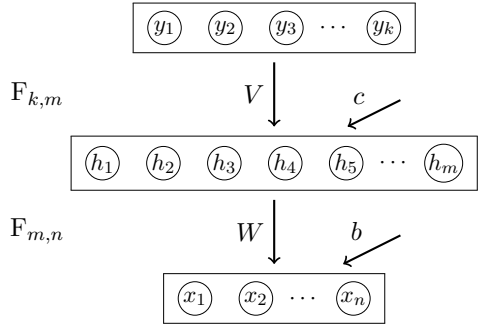
\begin{figure}
	\centering
	\tikzset{node distance=1.5cm, auto}
	\begin{tikzpicture}[scale=0.9, every node/.style={transform shape}]
	\node (L-0) at (0,0) {\unitlayer{x}{2}{n}}; 
	\node (L-1) at (0,2) {\unitlayer{h}{5}{m}}; 
	\node (L-2) at (0,4) {\unitlayer{y}{3}{k}}; 
	
	\draw[<-, line width = .8pt] (L-0) to node {$W$} (L-1);
	\draw[<-, line width = .8pt] (L-1) to node {$V$}  (L-2);
	
	\node (c) at (2,3) { };
	\draw[<-, line width = .8pt] (L-1) to node {$c$} (c);
	
	\node (b) at (2,1){ };
	\draw[<-, line width = .8pt] (L-0) to node {$b$} (b);
	
	\node (Q) at (-3.5,3) {$\F_{k,m}$};
	\node (R) at (-3.5,1) {$\F_{m,n}$};
	
	\end{tikzpicture}
	\caption{Feedforward network with a layer of $k$ input units, a layer of $m$ hidden units, and a layer of $n$ output units.}
	\label{figure:structure}
\end{figure}

The following set of probability distributions will be important in our analysis. 

\begin{definition}
	We denote by $\Ecal_m \subseteq\Delta_m$ the set of probability distributions on $\{0,1\}^m$ of the form 
	\begin{equation*}
		p(z) 
		= \frac{\exp(b^\top z)}{\sum_{z'\in\{0,1\}^m}\exp(b^\top z')}\quad\forall z\in\{0,1\}^m, 
	\end{equation*}
	parametrized by $b\in\R^m$. 
	This is precisely the set of probability distributions that factorize as 
	\begin{equation*}
		p(z) = \prod_{j\in[m]}p_j^{1-z_j} (1-p_j)^{z_j} \quad\forall  z=(z_1,\ldots, z_m)^\top\in\{0,1\}^m, 
	\end{equation*}
	for some $p_j\in (0,1)$ for all $j\in [m]$. 
\end{definition}

Note that each kernel $Q\in F_{k,m}$ represented by a feedforward layer is a tuple of $2^k$ factorizing probability distributions, $Q(\cdot|y)\in \Ecal_{m}$ for all $y\in\{0,1\}^k$. 

\subsection{Universal Approximation}

\begin{definition}
	A set $\Mcal\subseteq\Delta_{k,n}$ is a universal approximator if and only if every point from $\Delta_{k,n}$ can be approximated arbitrarily well by points from $\Mcal$. 
	This is the case if and only if $\overline{\Mcal}=\Delta_{k,n}$, 
	where $\overline{\Mcal}$ denotes the closure of $\Mcal$ in the Euclidean topology. 
\end{definition}

We will study the universal approximation properties of $\F_{k,m,n}$ in two cases.  
In the first case, only the first layer has free parameters, whereas the second layer has fixed parameters. 
This means that we consider the set $R\circ \F_{k,m}\subseteq\Delta_{k,n}$ for some fixed $R\in\F_{m,n}$.  
In the second case, all parameters are free. 

By comparing the number of free parameters and the dimension of $\Delta_{k,n}$, 
it is straightforward to obtain the following crude lower bound on the minimal number of hidden units that suffices for universal approximation: 
\begin{proposition}
	\label{proposition:minimal}
	Let $k\geq 1$ and $n\geq 1$. 
	\begin{itemize}
		\item 
		If there is an $R\in \overline{\F_{m,n}}$ with $R\circ \overline{\F_{k,m}} = \Delta_{k,n}$, 
		then $m\geq \lceil\frac{1}{(k+1)}2^k(2^n-1)\rceil$. 
		\item
		If $\overline{\F_{k,m,n}} = \Delta_{k,n}$, then $m\geq \lceil\frac{1}{(n+k+1)}(2^k(2^n-1)-n)\rceil$. 
	\end{itemize}
\end{proposition}

\subsection{Feedforward Deterministic Networks}

Deterministic networks are special cases of stochastic networks. 
Consider a feedforward stochastic network as defined above, but where all input weights and biases, $W$ and $b$, are multiplied by $r\in\R$. 
For generic choices of $W$ and $b$, when $r\to\infty$ each unit outputs $0$ or $1$ with probability one, depending on its inputs. 
In this case, the kernels represented by the feedforward network are deterministic, 
meaning that they have the form 
$P(x|y) = \delta_{f(y)}(x)$ for all $x\in\{0,1\}^n$, for all $y\in\{0,1\}^k$, for some function $f\colon \{0,1\}^k\to \{0,1\}^n$. 
The feedforward networks defined by these limits are called linear threshold networks. 

The representation of Boolean functions by linear threshold networks (with one binary output unit) has been studied extensively in the literature. 
The problem can be beautifully described as the problem of classifying subsets of vertices of the $k$-dimensional unit cube by an arrangement of oriented affine hyperplanes. 
In~\citep{Wenzel:2000:HAS:361159.361171} it was shown that, for $k\geq 2$ and $n=1$, the smallest number $m$ of hidden units for which a linear threshold network with one hidden layer can compute every Boolean function satisfies $2^{k/2} - \frac{k^2}{2} < -\frac{k^2}{2} + \sqrt{\frac{k^4}{4} +2^k} \leq m \leq \frac{3}{k+2} 2^k$. 

It is important to realize that, in the deterministic setting, 
if a network with $m$ hidden units can represent any Boolean function $f\colon \{0,1\}^k \to \{0,1\}$, 
then a network with $n\cdot m$ hidden units can represent any function $g\colon \{0,1\}^k \to \{0,1\}^n$. 
This is because one can always write $g = (f_1,\ldots, f_n)$ and compute the individual $f_i$'s in parallel groups of hidden and output units. 
In the stochastic setting the same is not true. 
In general, a joint distribution over $\{0,1\}^n$ cannot be written in terms of $n$ marginal distributions over $\{0,1\}$ alone, and so, the activities of the individual output units cannot be computed independently from each other.

\section{Results}
\label{section:results}

We bound the minimal number of hidden units of a universal approximator from above in two cases. 
In the first case we consider a network whose second layer has fixed weights and biases. 
In the second case all weights and biases are free parameters.   

\subsection{Fixed Weights in the Second Layer}

\begin{theorem}
	\label{theorem:first}
	Let $k\geq 1$ and $n\geq 1$. 
	There is an $R\in\overline{\F_{m,n}}$ such that $R\circ \overline{\F_{k,m}} = \Delta_{k,n}$, whenever $m\geq \frac{1}{2} 2^k(2^n-1)$. 
\end{theorem}
In view of the crude lower bound from Proposition~\ref{proposition:minimal}, 
this upper bound is tight at least when $k=1$.  

When there are no input units, $k=0$, we may set $\F_{0,m} =\Ecal_m$ and $\Delta_{0,n} =\Delta_n$. 
The theorem generalizes to this case as: 
\begin{proposition}
	\label{proposition:zero}
	Let $n\geq 2$. 
	There is an $R\in\overline{\F_{m,n}}$ with $R\circ \overline{\Ecal_m}=\Delta_n$, whenever $m\geq 2^{n}-1$. 
\end{proposition}
This bound is always tight, since the network uses exactly $2^n-1$ parameters to approximate every distribution from $\Delta_{n}$ arbitrarily well. 

\subsection{Changeable Weights in the Second Layer}
When both layers have changeable weights we obtain: 
\begin{theorem}
	\label{theorem:second}
	Let $k\geq 1$ and $n\geq 2$. 
	Then $\overline{\F_{k, m, n}}  = \Delta_{k,n}$, whenever $m\geq 2^{k-1}(2^{n-1}-1)$. 
\end{theorem}
Comparison with the crude lower bound from Proposition~\ref{proposition:minimal} reveals that this bound is tight at least for 
$(k,n)= (1,2), (2,2),(3,2), (1,3)$.

In the case of no inputs we obtain: 
\begin{proposition}
	Let $n\geq2$. Then $\overline{F_{m,n}}\circ \overline{\Ecal_m} = \Delta_n$, whenever $m\geq 2^{n-1}-1$. 
\end{proposition}

\subsection{Outline of the Proof}
Our strategy for proving Theorem~\ref{theorem:first} and Theorem~\ref{theorem:second} can be summarized as follows: 

\begin{itemize}
	\item 
	First we show that the first layer of $\F_{k,m,n}$ can approximate Markov kernels arbitrarily well, 
	which fix the state of some units, depending on the input, 
	and have an arbitrary product distribution over the states of the other units. 
	The idea is illustrated in Fig.~\ref{figure:proofidea}. 
	\item 
	Then we show that the second layer can approximate deterministic kernels arbitrarily well, whose rows are copies of all point measures from $\Delta_n$, ordered in a good way with respect to the different inputs. 
	Note that the point measures are the vertices of the simplex $\Delta_n$. 
	\item 
	Finally, we show that the set of product distributions of each block of hidden units is mapped to the convex hull of the rows of the kernel represented by the second layer, which is $\Delta_n$. 
	\item 
	The output distributions of distinct sets of inputs is modeled individually by distinct blocks of hidden units and so we obtain the universal approximation of Markov kernels. 
\end{itemize}

The goal of our analysis is to construct the individual pieces of the network as compact as possible. 
In particular, Lemma~\ref{lemma:firstlayer} will provide a trick that allows us to use each block of units in the hidden layer for a pair of distinct input vectors at the same time. 
This allows us to halve the number of hidden units that would be needed if each input had an individual block of active hidden units. 
Similarly, Lemma~\ref{lemma:deterministicgeometric2} will provide a trick for producing more flexible mixture components at the output layer than simply point measures. This again allows us to halve the number of hidden units of the simpler construction. 

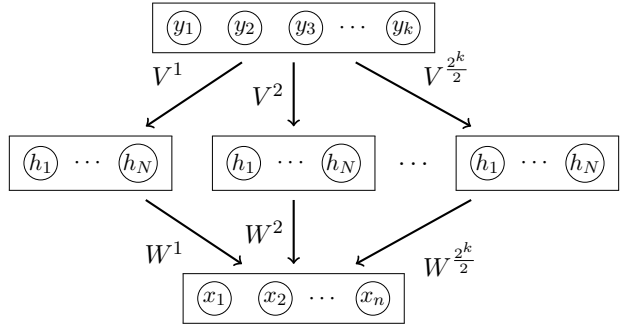
\begin{figure}
	\centering
	\tikzset{node distance=1.5cm, auto}
	\begin{tikzpicture}[scale=0.9, every node/.style={transform shape}]
	\node (L-0) at (0,0) {\unitlayer{x}{2}{n}}; 
	\node (L-1a) at (-3,2) {\unitlayer{h}{1}{N}}; 
	\node (L-1b) at (0,2) {\unitlayer{h}{1}{N}}; 
	\node (L-1d) at (1.8,2) {$\dots$}; 
	\node (L-1c) at (3.6,2) {\unitlayer{h}{1}{N}}; 
	\node (L-2) at (0,4) {\unitlayer{y}{3}{k}}; 
	
	\draw[<-, line width = .8pt] (L-0) to node {$W^1$} (L-1a);
	\draw[<-, line width = .8pt] (L-0) to node {$W^2$} (L-1b);
	\draw[<-, line width = .8pt] (L-0) to node [swap] {$W^{\frac{2^k}{2}}$} (L-1c);
	\draw[<-, line width = .8pt] (L-1a) to node {$V^1$}  (L-2);
	\draw[<-, line width = .8pt] (L-1b) to node {$V^2$}  (L-2);
	\draw[<-, line width = .8pt] (L-1c) to node [swap] {$V^{\frac{2^k}{2}}$}  (L-2);
	
	\end{tikzpicture}
	\caption{Illustration of the construction used in our proof. 
		Each block of hidden units is active on a distinct subset of possible inputs. 
		The output layer integrates the activities of the block that was activated by the input, 
		and produces corresponding activities of the output units. }
	\label{figure:proofidea}
\end{figure}

\section{The Number of Hidden Units for Fixed Weights in the Second Layer}
\label{section:proof1}

\subsection{The First Layer}
We start with the following lemma. 

\begin{lemma} 
	\label{lemma:firstlayer}
	Let $y',y''\in\{0,1\}^k$ differ only in one entry, and let $q',q''$ be any two distributions on $\{0,1\}$. 
	Then $F_{k,1}$ can approximate the following arbitrarily well: 
	\begin{equation*}
		p(\cdot|y) = 
		\left\{
		\begin{array}{l l}
			q', & \text{if $y=y'$}\\
			q'', & \text{if $y=y''$}\\
			\delta_{0}, & \text{else}
		\end{array}
		\right.. 
	\end{equation*}
\end{lemma}
\begin{proof}
	Given the input weights and bias, $V\in\R^{1\times k}$ and $c\in\R$, 
	for each input $y\in\{0,1\}^k$ the output probability is given by 
	\begin{equation}
		p(z=1|y) = \sigma(V y + c).
		\label{eq:output}
	\end{equation} 
	Since the two vectors $y',y''\in\{0,1\}^k$ differ only in one entry, they are the vertices of an edge $E$ of the $k$-dimensional unit cube. 
	Let $l\in[k]:=\{1,\ldots,k \}$ be the entry in which they differ, with $y'_l=0$ and $y''_l=1$. 
	Since $E$ is a ($1$-dimensional) face of the cube, there is a supporting hyperplane of $E$. 
	This means that there are $\tilde V\in\R^{1\times k}$ and $\tilde c\in\R$ with 
	$\tilde V y + \tilde c =0$ if $y\in E$, 
	and $\tilde V y + \tilde c < -1$ if $y\in\{0,1\}^k\setminus E$. 
	Let $s'= \sigma^{-1}(q'(z=1))$ and $s''=\sigma^{-1}(q''(z=1))$. 
	We define $c=\alpha \tilde c + s'$ and $V=\alpha\tilde V + (s'' - s')e_l$. 
	Then, as $\alpha\to\infty$, 
	\begin{equation*}
		V y +c = 
		\left\{
		\begin{array}{l l}
			s', & \text{if $y=y'$}\\
			s'', & \text{if $y=y''$}\\
			-\infty, & \text{else}
		\end{array}
		\right.. 
	\end{equation*}  
	Plugging this into~\eqref{eq:output} proves the claim. 
\end{proof}

Given any binary vector $y\in\{0,1\}^k$, 
let $\dec(y):=\sum_{i=1}^k 2^{i-1}y_i$ be its integer representation. 
Using the previous lemma, we obtain the following.

\begin{proposition}
	\label{proposition:firstla}
	Let $N\geq 1$ and $m=2^{k-1} N$. 
	For each $y\in\{0,1\}^k$, let $p(\cdot|y)$ be an arbitrary distribution from $\Ecal_{N}$. 
	The model $\F_{k,m}$ can approximate the following kernel from $\Delta_{k,m}$ arbitrarily well: 
	\begin{align*}
		P(h | y )  = &
		\delta_{\bo}(h^0 )   \cdots  
		\delta_{\bo}(h^{\lfloor\dec(y)/2\rfloor-1}) 
		p(h^{\lfloor\dec(y)/2\rfloor}|y) \\
		&
		\times
		\delta_{\bo}(h^{\lfloor\dec(y)/2\rfloor+1}) \cdots
		\delta_{\bo}(h^{2^{k-1} -1}), 
	\end{align*}
	where $h^i= (h_{Ni+1},\ldots, h_{N(i+1)})$ for all $i\in\{0,1,\ldots,2^{k-1}-1\}$. 
\end{proposition}

\begin{proof}
	We divide the set $\{0,1\}^k$ of all possible inputs into $2^{k-1}$ disjoint pairs with successive decimal values. 
	The $i$-th pair consists of the two vectors $y$ with $\lfloor \dec(y)/2 \rfloor = i$, for all $i\in\{0,\ldots,2^{k-1}-1 \}$. 
	The kernel $P$ has the property that, for the $i$-th input pair, 
	all output units are inactive with probability one, 
	except those with index $Ni+1, \ldots, N(i+1)$. 
	Given a joint distribution $q$ on the states of $I$ units, 
	let $q_{j}$ denote the corresponding marginal distribution on the states of the $j$-th of these units. 
	By Lemma~\ref{lemma:firstlayer}, we can set 
	\begin{equation*}
		P_{Ni+j}(\cdot|y) = 
		\left\{
		\begin{array}{l l}
			p_j(\cdot|y) , & \text{if $\dec(y)= 2i  $}\\
			p_j(\cdot|y) , & \text{if $\dec(y)= 2i+1$}\\
			\delta_{0}, & \text{else}
		\end{array}
		\right.  
	\end{equation*}
	for all $i\in\{0,\ldots,2^{k-1}-1 \}$ and $j\in[N]$. 
\end{proof}

\subsection{The Second Layer}

\begin{figure}
	\centering
	\includegraphics{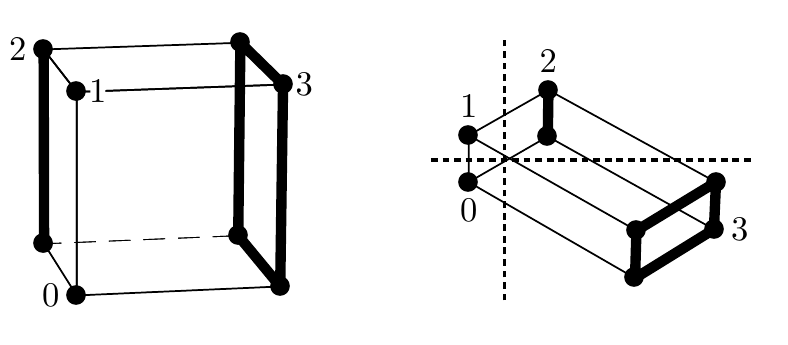}
	\caption{ %
		Illustration of Lemma~\ref{lemma:geometricclassification} for $n=2$, $N=2^n-1 = 3$. 
		There is an arrangement of $n$ hyperplanes which divides the vertices of a $(2^n-1)$-dimensional cube as $0|1|2,3|4,5,6,7|\cdots|2^{2^n-2},\ldots,2^{2^n-1}-1$. 
	}
	\label{figure:geometricclassification}
\end{figure}

For the second layer we will consider deterministic kernels. 
Given a binary vector $z$, 
let $l(z):=\lceil \log_2(\dec(z)+1)\rceil$ denote the largest $j$ with $z_j=1$. 
Here we set $l(0,\ldots,0)=0$. 
Given an integer $l\in\{0,\ldots, 2^n-1 \}$, let $\bin_n(l)$ denote the $n$-bit binary representation of $l$; 
that is, the vector with $\dec(\bin_n(l))=l$. 

\begin{lemma}
	\label{lemma:deterministicgeometric}
	Let $N=2^n-1$. The set $\F_{N,n}$ can approximate the following deterministic kernel arbitrarily well: 
	\begin{equation*}
		Q(\cdot | z) = \delta_{\bin_n l(z)}(\cdot) \quad \forall z\in\{0,1\}^N. 
	\end{equation*}
\end{lemma}
In words, the $z$-th row of $Q$ indicates the largest non-zero entry of the binary vector $z$. 
For example, for $n=2$ we have $N=3$ and 
\begin{equation*}
	Q = 
	\begin{blockarray}{ccccc}
		00 & 01 & 10 & 11 & \\
		\begin{block}{(cccc) c}
			1 &    &      &      & 000\\
			& 1 &      &      & 001\\
			&    & 1   &      & 010\\
			&    & 1   &      & 011\\
			&    &      &  1  & 100\\
			&    &      &  1 & 101\\
			&    &      &  1 & 110\\
			&    &      &  1 & 111\\
		\end{block}
	\end{blockarray}. 
\end{equation*}
\begin{proof}{\bf of Lemma~\ref{lemma:deterministicgeometric}}
	Given the input and bias weights, $W\in\R^{n\times N}$ and $b\in\R^{n}$, for each input $z\in\{0,1\}^N$ the output distribution is the product distribution $p(\cdot|z)\in\Ecal_n$ with exponential parameters $W z + b$. 
	If $\sgn(W z + b) = \sgn(x-\tfrac12)$ for some $x\in\{0,1\}^n$, 
	then the product distribution with parameters $\alpha(Wz+b)$, $\alpha\to\infty$ tends to $\delta_x$. 
	We only need to show that there is a choice of $W$ and $b$ with 
	$\sgn(W z +b) = \sgn(f(z) -\tfrac12)$, $f(z)= \bin_n \lceil \log_2(\dec(z)+1)\rceil$, for all $z\in\{0,1\}^N$. 
	That is precisely the statement of Lemma~\ref{lemma:geometricclassification}. 
\end{proof}

We used the following lemma in the proof of Lemma~\ref{lemma:deterministicgeometric}. 
For $l=0,1,\ldots, 2^n-1$, the $l$-th orthant of $\R^n$ is the set of all vectors $r\in\R^n$ with strictly positive or negative entries and $\dec(\hs(r)) = l$, where 
$\hs$ denotes the entry-wise Heaviside function, assigning value $0$ to negative entries and value $1$ to positive entries.

\begin{lemma}
	\label{lemma:geometricclassification}
	Let $N=2^n-1$. 
	There is an affine map $\{0,1\}^N\to \R^n$; $z\mapsto W z + b$, sending $\{z\in\{0,1\}^N\colon l(z)=l \}$ to the $l$-th orthant of $\R^n$, for all $l\in\{0,1,\ldots,N \}$. 
\end{lemma}
\begin{proof}
	Consider the affine map 
	$z\mapsto Wz+b$, 
	where $b=-(1,\ldots,1)^\top$ and the $l$-th column of $W$ is $2^{l+1}(\bin_n(l)-\frac{1}{2})$ for all $l\in\{1,\ldots, N\}$. 
	For this choice, $\sgn(W z + b) = \sgn( \bin_n(l(z)) -\frac{1}{2} )$ lies in the $l$-th orthant of $\R^n$. 
\end{proof}

Lemma~\ref{lemma:geometricclassification} is illustrated in Fig.~\ref{figure:geometricclassification} for $n=2$. 
As another example, for $n=3$ the affine map can be defined as  
$z\mapsto W z + b$, where 
\begin{equation*}
	b =
	\begin{pmatrix}
		-1 \\ -1 \\ -1
	\end{pmatrix},
	\quad
	\text{and}
	\quad
	W=
	\left(
	\begin{array}{rrrrrrr}
		2 & -4 & 8   & -16 & 32  & -64 & 128\\
		-2 &   4 & 8   & -16 & -32& 64   & 128\\
		-2 & -4 & - 8&   16 & 32  & 64   & 128
	\end{array}
	\right).
\end{equation*}

\begin{proposition}
	\label{proposition:univdetmap}
	Let $N=2^n -1$ and let $Q$ be defined as in Lemma~\ref{lemma:deterministicgeometric}. 
	Then $Q\circ \Ecal_{N} = \Delta_{n}^+$. 
\end{proposition}
\begin{proof}
	Consider a strictly positive product distribution $p\in \Ecal_N$ with $p(z)=\prod_i p_i^{1-z_i}(1-p_i)^{z_i}$ for all $z\in\{0,1\}^N$. 
	Then $p^\top Q\in\Delta_n$ is the vector $q=(q_0,q_1,\ldots, q_N)$ with entries 
	\begin{eqnarray*}
		q_i 
		&=& \sum_{z\colon l(z)=i}p(z)\\
		&=& \sum_{z_k, k<i} \Big( \prod_{k<i} p_k^{1-z_k}(1-p_k)^{z_k}\Big) (1-p_i) \Big( \prod_{j>i}p_j\Big)\\
		&=& (1-p_i) \prod_{j>i}p_{j}
	\end{eqnarray*}
	for all $i=1,\ldots, N$, and $q_0=\prod_{j>0}p_j$. 
	Therefore, 
	\begin{equation}
		\frac{q_i}{q_0} = \frac{1-p_i}{p_i} \frac{1}{\prod_{j=1}^{i-1}p_j} \quad\forall i=1,\ldots,N. 
		\label{eq:quotients}
	\end{equation}
	Since $\frac{1-p_i}{p_i}$ can be made arbitrary in $(0,\infty)$ by choosing an appropriate $p_i$, 
	independently of $p_j$, for $j<i$,  
	the quotient $\frac{q_i}{q_0}$ can be made arbitrary in $(0,\infty)$ for all $i\in\{1,\ldots, N\}$. This implies that $q$ can be made arbitrary in $\Delta_n^+$. 
\end{proof}

\begin{proof}{\bf of Theorem~\ref{theorem:first}}
	This follows from Proposition~\ref{proposition:firstla} and Proposition~\ref{proposition:univdetmap}. 
\end{proof}

\section{The Number of Hidden Units for Changeable Weights in the Second Layer}
\label{section:proof2}

In order to prove Theorem~\ref{theorem:second} we will use the same construction of the first layer as in the previous section. 
For the second layer we will use the following refinement of Lemma~\ref{lemma:deterministicgeometric}. 

\begin{lemma}
	\label{lemma:deterministicgeometric2}
	Let $n\geq 2$ and $N=2^{n-1}-1$. 
	The set $\F_{N,n}$ can approximate the following kernels arbitrarily well: 
	\begin{equation*}
		Q(\cdot | z) = \lambda_z \delta_{\bin_n 2 l(z) }(\cdot) + (1-\lambda_z)\delta_{\bin_n 2 l(z) + 1}(\cdot) \quad \forall z\in\{0,1\}^N, 
	\end{equation*}
	where $\lambda_z$ are certain (not mutually independent) weights in $[0,1]$. 
	Given any $r_l\in\R_+$ for all $l\in\{0,1, \ldots, N \}$, 
	it is possible to choose the $\lambda_z$'s such that 
	\begin{equation*}
		\frac{\sum_{z\colon l(z) = l} \lambda_z}
		{\sum_{z\colon l(z) = l} ( 1 - \lambda_z)} = r_l\quad\forall l\in\{0,1,\ldots, N  \}. 
	\end{equation*}  
\end{lemma}

In words, the $z$-th row of $Q$ is a convex combination of the indicators of $2l(z)$ and $2l(z)+1$, 
and, furthermore, the total weight assigned to $2l$ relative to the total weight assigned to $2l+1$ can be made arbitrary for each $l$. 
For example, for $n=2$ we have $N=1$ and 
\begin{equation*}
	Q = 
	\begin{blockarray}{ccccc}
		00 & 01 & 10 & 11 & \\
		\begin{block}{(cccc) c}
			\lambda_0 & (1-\lambda_0)&      		&      & 0\\
			&					    &\lambda_1& (1-\lambda_1)     & 1\\
		\end{block}
	\end{blockarray}. 
\end{equation*}
The sum of all weights in a given even column can be made arbitrary, relative to the sum of all weights in the column right next to it, for all $N+1$ such pairs of columns simultaneously. 
\begin{proof}{\bf of Lemma~\ref{lemma:deterministicgeometric2}}
	Consider the sets $Z_l=\{z\in\{0,1\}^N\colon l(z)=l \}$, for $l=0,1,\ldots, N$. 
	Let $W'\in\R^{(n-1)\times N}$ and $b'\in\R^{n-1}$ be the input weights and biases defined in Lemma~\ref{lemma:deterministicgeometric}. 
	We define $W$ and $b$ by appending a row $(\mu_1,\ldots,\mu_N )$ on top of $W'$ and an entry $\mu_0$ on top of $b'$. 
	
	If $\mu_j<0$ for all $j=0,1,\ldots, N$, then $z\mapsto W z + b$ maps $Z_l$ to the $2l$-th orthant of $\R^n$, for each $l=0,1,\ldots, N$. 
	
	Consider now some arbitrary fixed choice of $\mu_j$, $j< l$. 
	Choosing $\mu_l<0$ with $|\mu_l|>\sum_{j<l}|\mu_l|$, $Z_l$ is mapped to the $2l$-th orthant. 
	If $\mu_l \to -\infty$, then $\lambda_z \to 1$ for all $z$ with $l(z)=l$. 
	As we increase $\mu_l$ to a sufficiently large positive value, the elements of $Z_l$ gradually are mapped to the $(2l + 1)$-th orthant. 
	If $\mu_l\to\infty$, then $(1 - \lambda_z)\to 1$ for all $z$ with $l(z)=l$. 
	By continuity, there is a choice of $\mu_l$ such that $\frac{\sum_{z\colon l(z) =l} \lambda_z}{ \sum_{z\colon l(z) = l} (1-\lambda_z)} = r_l$. 
	
	Note that the images of $Z_j$, $j<l$, are independent of the $i$-th rows of $W$ for all $i=l,\ldots, N$. 
	Hence changing $\mu_l$ does not have any influence on the images of $Z_l$ nor on $\lambda_z$ for $z\colon l(z)<l$. 
	Tuning $\mu_i$ sequentially, starting with $i=0$, we obtain a kernel that approximates any $Q$ of the claimed form arbitrarily well. 
\end{proof}

Let $\Q_n^N$ be the collection of kernels described in Lemma~\ref{lemma:deterministicgeometric2}. 
\begin{proposition}
	\label{proposition:sendod}
	Let $n\geq 2$ and $N=2^{n-1} -1$. 
	Then $\Q_n^N\circ \Ecal_{N} = \Delta_{n}^+$. 
\end{proposition}

\begin{proof}
	Consider a strictly positive product distribution $p\in \Ecal_N$ with $p(z) = \prod_{i=1}^N p_i^{z_i-1}(1-p_i)^{z_i}$ for all $z\in\{0,1\}^N$. 
	Then $p^\top Q\in\Delta_n$ is a vector $(q_0,q_1,\ldots, q_{2N+1})$ whose entries satisfy 
	\begin{eqnarray*}
		q_{2i}+q_{2i+1} = (1-p_i) \prod_{j>i}p_{j}, 
	\end{eqnarray*}
	for all $i = 1, \ldots, N$ and $q_0 + q_1 = \prod_{j>i}p_j$. 
	As in the proof of Proposition~\ref{proposition:univdetmap}, this implies that the vector $(q_0+q_1, q_2+q_3,\ldots, q_{2N}+q_{2N+1})$ can be made arbitrary in $\Delta_{n-1}^+$. 
	This is irrespective of the coefficients $\lambda_0,\ldots, \lambda_N$. 
	Now all we need to show is that we can make $q_{2i}$ arbitrary relative to $q_{2i+1}$ for all $i=0,\ldots, N$. 
	
	We have 
	\begin{eqnarray*}
		q_{2i} 
		&=&\sum_{z\colon l(z)=i} \lambda_z p(z)\\
		&=& \left(\sum_{z\colon l(z)=i}\lambda_{z} \Big( \prod_{k<i} p_k^{1-z_k}(1-p_k)^{z_k}\Big)\right)\\
		& & \quad \times (1-p_i) \Big( \prod_{j>i}p_j\Big)
	\end{eqnarray*}
	and
	\begin{eqnarray*}
		q_{2i+1} 
		&=&\sum_{z\colon l(z)=i} (1 - \lambda_z) p(z)\\
		&=& \left(\sum_{z\colon l(z)=i}(1 -\lambda_{z}) \Big( \prod_{k<i} p_k^{1-z_k}(1-p_k)^{z_k}\Big)\right) \\
		& & \quad \times (1-p_i) \Big( \prod_{j>i}p_j\Big). 
	\end{eqnarray*}
	Therefore, 
	\begin{eqnarray*}
		\frac{q_{2i}}{q_{2i+1}} 
		&=& 
		\frac{ \sum_{z\colon l(z)=i}\lambda_{z} \Big( \prod_{k<i} p_k^{1-z_k}(1-p_k)^{z_k}\Big) }
		{\sum_{z\colon l(z)=i}(1 -\lambda_{z}) \Big( \prod_{k<i} p_k^{1-z_k}(1-p_k)^{z_k}\Big)}. 
	\end{eqnarray*}
	By Lemma~\ref{lemma:deterministicgeometric2} it is possible to choose all $\lambda_z$ arbitrarily close to zero for all $z$ with $l(z)=i$ and have them transition continuously to values arbitrarily close to one (independently of the values of $\lambda_z$, $z\colon l(z)\neq i$). 
	Since all $p_k$ are strictly positive, this implies that the quotient $\frac{q_{2i}}{q_{2i+1}}$ takes all values in $(0,\infty)$ as the $\lambda_z$, $z\colon l(z)=i$ transition from zero to one. 
\end{proof}

\begin{proof}{\bf of Theorem~\ref{theorem:second}}
	This follows from Proposition~\ref{proposition:firstla} and Proposition~\ref{proposition:sendod}. 
\end{proof}

\section{Conclusions}
\label{section:conclusion}

This article proves upper bounds on the minimal size of binary shallow stochastic feedforward networks with sigmoid activation probabilities that can approximate any stochastic function with a given number of binary inputs and outputs arbitrarily well. 
By our analysis, if all parameters of the network are free, $2^{k-1}(2^{n-1}-1)$ hidden units suffice, and, if only the parameters of the first layer are free, $2^{k-1}(2^n-1)$ hidden units suffice. 

It is interesting to compare these results with what is known about universal approximation of Markov kernels by shallow undirected stochastic networks, called conditional restricted Boltzmann machines. 
For those networks previous work~\citep{montufar2014expressive} has shown that $2^{k-1}(2^n-1)$ hidden units suffice, whereby, if the number $k$ of input units is large enough, $\frac14 2^{k}(2^n-1 + 1/30)$ suffice. 
These bounds are sandwiched between our bounds for feedforward networks. 
In the case of no input units, our bound $2^{n-1}-1$ equals the known bounds for universal approximation of probability distributions by restricted Boltzmann machines~\citep{Montufar2011}. 
Hence, 
given the current state of knowledge, if we were to specify a smallest possible universal approximator of Markov kernels or probability distributions, feedforward networks would seem preferable. 
However, verifying the tightness of the bounds appears to be a very challenging problem in either case. 
It has been observed that undirected networks can represent many kernels that can be represented by feedforward networks, especially when these are not too stochastic~\citep{montufar2014expressive,montufar2014deep}. 
In future work it would be interesting to compare the representational power of both network architectures in more detail. 

We think that it is possible to adapt our analysis to cover deep architectures as well. 
This should allow us to conclude that a multilayer feedforward stochastic network with $k$ input units and $n$ output units is a universal approximator of Markov kernels if it has about $2^{n-1}$ hidden layers, each containing about $n 2^{k-1}$ units.  The verification of this claim is left for future work. 
In relation with this, the results presented in this paper should be helpful for analyzing the relative representational power of shallow vs. deep stochastic feedforward networks, a topic that has attracted much interest in recent years and that still poses a great many questions.

\acks{I would like to thank the Santa Fe Institute, Santa Fe, NM, USA, and the RIKEN Brain Science Institute, Hirosawa, Saitama, Japan, for hosting me during the work on this article.}

\bibliography{referenzen}

\end{document}